\documentclass[a4paper,superscriptaddress,preprintnumbers,floatfix,letter]{elsarticle}

\usepackage{amssymb}
\usepackage{amsthm}

\usepackage{xcolor}
\usepackage{amscd}
\usepackage{amsmath}
\usepackage{bm}
\usepackage{comment}
\usepackage{hyperref}  
\usepackage{mathrsfs}  
\usepackage{enumitem}  
\usepackage{subcaption}  
\usepackage{tikz}
\usetikzlibrary{arrows.meta} 
\usepackage[outline]{contour} 
\contourlength{1.4pt}
\tikzset{
  >=latex, 
  node/.style={thick,circle,draw=myblue,minimum size=22,inner sep=0.5,outer sep=0.6},
  node in/.style={node,green!20!black,draw=mygreen!30!black,fill=mygreen!25},
  node hidden/.style={node,blue!20!black,draw=myblue!30!black,fill=myblue!20},
  node convol/.style={node,orange!20!black,draw=myorange!30!black,fill=myorange!20},
  node out/.style={node,red!20!black,draw=myred!30!black,fill=myred!20},
  connect/.style={thick,mydarkblue}, 
  connect arrow/.style={-{Latex[length=4,width=3.5]},thick,mydarkblue,shorten <=0.5,shorten >=1},
  node 1/.style={node in}, 
  node 2/.style={node hidden},
  node 3/.style={node out}
}

\colorlet{myred}{red!80!black}
\colorlet{myblue}{blue!80!black}
\colorlet{mygreen}{green!60!black}
\colorlet{myorange}{orange!70!red!60!black}
\colorlet{mydarkred}{red!30!black}
\colorlet{mydarkblue}{blue!40!black}
\colorlet{mydarkgreen}{green!30!black}

\usepackage[figuresright]{rotating}

\DeclareFontFamily{T1}{calligra}{}
\DeclareFontShape{T1}{calligra}{m}{n}{<->s*[1.44]callig15}{}
\DeclareMathAlphabet\mathnn       {T1}{pzc} {m} {n}

\newtheorem{thm}{Theorem}[section]
\newtheorem{prop}{Proposition}[section]
\newtheorem{defin}{Definition}

\newtheorem{cor}{Corollary}[section]

\begin{document}

\begin{frontmatter}




\title{Plateaus, Optima, and Overfitting in Multi-Layer Perceptrons: A Saddle-Saddle-Attractor Scenario}


\author[l1]{Alex Alì Maleknia\footnote{Present address: IMAG, IROKO, Univ Montpellier, Inria, CNRS}}
\ead{alex.maleknia@inria.fr}

\author[l1,l2]{Yuzuru Sato}  
\ead{ysato@math.sci.hokudai.ac.jp}
\address[l1]{Department of Mathematics, Hokkaido University,  Kita 12 Nishi 7, Kita-ku, Sapporo, Hokkaido 060-0812, Japan} 
\address[l2]{RIES, Hokkaido University,  Kita 12 Nishi 7, Kita-ku, Sapporo, Hokkaido 060-0812, Japan} 

\begin{abstract}
Vanishing gradients and overfitting are central problems in machine learning, yet are typically analyzed in asymptotic regimes that obscure their dynamical origins. Here we provide a dynamical description of learning in multi-layer perceptrons (MLPs) via a minimal model inspired by Fukumizu and Amari. We show that training dynamics traverse plateau and near-optimal regions, both organized by saddle structures, before converging to an overfitting regime. Under suitable conditions on the data, this regime collapses to a single attractor modulo symmetry. 
Furthermore, for finite noisy datasets, convergence to the theoretical optimum is impossible, and the dynamics necessarily settle into an overfitting solution.
\end{abstract}


\end{frontmatter}


\section{Introduction}

Over the past two decades, substantial effort has been devoted to improving gradient-based training of neural networks \cite{Kingma2014AdamAM,jordan2024muon,AD_for_plateaus,plateau_relu}. Here we analyse a minimal dynamical scenario capturing two central issues: vanishing gradients and overfitting in gradient descent for multi-layer perceptrons. Motivated by \cite{Amari,optimal_structure,Sato}, we consider a Fukumizu\textendash Amari model with observational noise and study its learning dynamics from a dynamical systems perspective. Our main result (Theorem \ref{th:main}) shows that, except for a set of measure zero and with high probability, all trajectories converge to overfitting solutions when the number of data points is large or the data variance is small. We further conjecture that the transient dynamics are organized by numerous saddle structures \cite{zhang2026saddletosaddle}, including plateaus and near-optimal states. Based on this observation, we propose that a generic saddle-saddle-attractor scenario governs these classes of machine learning systems.

The structure of the paper is as follows.
After introducing the general notation and definitions in Section \ref{sec:prelim}, we present the main theoretical analysis concerning the overfitting region and the optimal region in Section \ref{sec:thms}.
To complement the analysis, Section \ref{sec:min_model} provides numerical experiments conducted within the minimal gradient descent model.
The purpose is to illustrate the vanishing gradient phenomenon and overfitting in their essential forms.
The final section summarizes the results and outlines directions for future research.

\section{Preliminaries}\label{sec:prelim}

\subsection{Gradient descent for multi-layer perceptrons}


A \textit{multi-layer perceptron} (MLP) \cite{Universal_NN,Bishop} consists of a system of interconnected nodes (called \textit{neurons}), typically represented as a computational graph, as illustrated in Fig.~\ref{fig:MLP}.

\begin{figure}[hbt]
    \centering
    \begin{tikzpicture}[x=2cm,y=1cm]
        \message{^^JNeural network activation}
        \def\NI{2} 
        \def\Na{4} 
        \def\Nb{5} 
        \def\NO{1} 
        \def\yshift{0.5} 
        
        \node[node in,outer sep=0.6] (NI-1) at (0,0) {$1$};
        \node[node in,outer sep=0.6] (NI-2) at (0,-2) {$x$};
        
        \foreach \i [evaluate={\c=int(\i==\Na); \y=\Na/2-\i-\c*\yshift; \index=(\i<\Na?int(\i):"m");}]
            in {\Na,...,1}{ 
            \ifnum\i>1
            \node[node hidden]
            (Na-\i) at (1,\y) {$\sigma$};
            \foreach \j in {1,...,\NI}{ 
            \draw[connect arrow,line width=1.2] (NI-\j) -- (Na-\i);
            }
            \else
            \node[node in]
            (Na-\i) at (1,\y) {$1$};
            \fi
            }
        
        \foreach \i [evaluate={\c=int(\i==\Nb); \y=\Nb/2-\i-\c*\yshift; \index=(\i<\Nb?int(\i):"m");}]
            in {\Nb,...,1}{ 
            \ifnum\i>1
            \node[node hidden]
            (Nb-\i) at (2,\y) {$\sigma$};
            \foreach \j in {1,...,\Na}{ 
            \draw[connect arrow,line width=1.2] (Na-\j) -- (Nb-\i);
            }
            \else
            \node[node in]
            (Nb-\i) at (2,\y) {$1$};
            \fi
            }

        \node [node out] (1) at (3, -1/2) {$y$};
        \foreach \j in {1,...,\Nb}{ 
        \draw[connect arrow,line width=1.2] (Nb-\j) -- (1);
        }
        
        \path (Na-\Na) --++ (0,1+\yshift) node[midway,scale=1] {$\vdots$};
        \path (Nb-\Nb) --++ (0,1+\yshift) node[midway,scale=1] {$\vdots$};
    
        \node[above=5,align=center,mygreen!60!black] at (0,1.3) {input\\[-0.2em]layer};
        \node[above=2,align=center,myblue!60!black] at (1.5,1.8) {hidden layers};
        \node[above=8,align=center,myred!60!black] at (3,1.3) {output\\[-0.2em]layer};
    \end{tikzpicture}
    \caption{A multi-layer perceptron with two hidden layers of arbitrary size.}
    \label{fig:MLP}
\end{figure}
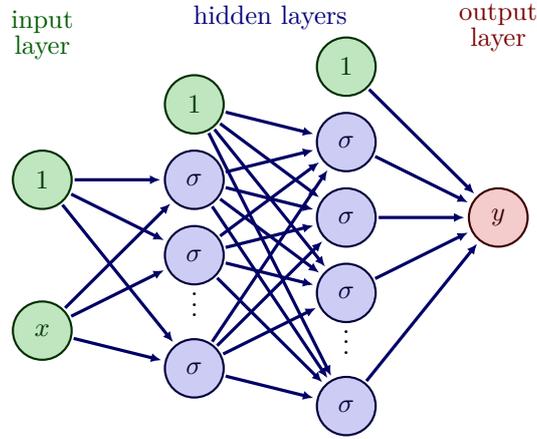

An MLP $\mathnn{N}$ naturally defines a function $f_\mathnn{N}:\mathbb{R}^{d_{in}}\to\mathbb{R}^{d_{out}}$, which depends on the architecture, the activation function $\sigma:\mathbb{R}\to\mathbb{R}$, and the weights of the edges of $\mathnn{N}$.\\ 
For the sake of simplicity, we will focus our efforts on the case of networks with only one hidden layer (3-layer perceptrons). 
In this setting, we will use the following notation:
\begin{itemize}
    \item $m$: number of neurons in the hidden layer;
    \item $\boldsymbol{w}=(w_{i,j})$: weights connecting the input layer to the hidden layer for $i=1,\dots,d_{in}$, $j=1,\dots,m$
    \item $\boldsymbol{v}=(v_{j,k})$: weights connecting the hidden layer to the output layer for $j=1,\dots,m$, $k=1,\dots,d_{out}$;
    \item $\boldsymbol{b}=(b_{1,j},b_{2,k})$: weights of the constant $1$ in the first and second layer (called \textit{biases}) for $j=1,\dots,m$, $k=1,\dots,d_{out}$;
    \item $\boldsymbol{\theta}=(\boldsymbol{v},\boldsymbol{w},\boldsymbol{b})$: the parameter of $\mathnn{N}$;
    \item $\boldsymbol{\Theta}_m=\mathbb{R}^{m(d_{in}+d_{out}+1)+d_{out}}$: the space of all possible parameters. 
\end{itemize}
With the above notation, we can write $f_{\mathnn{N}}$ explicitly as:
\begin{equation*}
    f_{\mathnn{N}}(x)=f(x;\boldsymbol{\theta})=\sum_{i=1}^m v_i\,\sigma(w_i\cdot x+b_{1,i})+b_2 \, .
\end{equation*}

Given a dataset $D^n=\left\{(x_i,y_i)\right\}_{i=1}^n$, we want to find the best fitting function for the data among the $m$-neuron MLP.
This problem is equivalent to finding the best parameter $\boldsymbol{\theta}\in\boldsymbol{\Theta}_m$ such that $f(x;\boldsymbol{\theta})$ minimizes the 
\textit{training error}:
\begin{equation}
    L(\boldsymbol{\theta};D^n)=\frac{1}{2n}\sum_{i=1}^n h^2(x_i,y_i;\boldsymbol{\theta}),
\end{equation}
where $h(x,y;\boldsymbol{\theta})=\|f(x;\boldsymbol{\theta})-y\|_2$.\\
In order to find the best parameter, we will use the gradient descent algorithm.
Fixed an initial condition $\boldsymbol{\theta}_0\in\boldsymbol{\Theta}_m$, the update rule is:
\begin{equation}\tag{GD}\label{eq:GD}
    \boldsymbol{\theta}(t+1)=\boldsymbol{\theta}(t)-\eta\nabla_{\boldsymbol{\theta}}L(\boldsymbol{\theta}(t);D^n)
\end{equation}
for each $t\in\mathbb{N}$.


\subsection{Vanishing gradient}
In many settings throughout machine learning \cite{pascanu13, plateau_relu, Montanari,Amari,AD_for_plateaus}, it has been observed that the training process of a neural network can be slowed down when the gradient of the loss function remains close to zero for a long period, before suddenly increasing. 
This phenomenon is referred to as the \textit{vanishing gradient}, and the resulting slow dynamics is known as the \textit{plateau phenomenon}.
The underlying causes of this behaviour in the general setting remain unclear. 
The study reported in \cite{Amari} shows that, for an MLP with the hyperbolic tangent activation function, one possible cause of the vanishing gradient is that the parameter $\boldsymbol{\theta}$ approaches a \textit{singular region} where the network becomes reducible \cite{optimal_structure}; that is, the function $f(x;\boldsymbol{\theta})$ can be represented by a strictly smaller network. 
In addition, they propose a minimal model, which we refer to as the Fukumizu–Amari model, to investigate the plateau phenomenon.
This model serves as the main inspiration for our model: 
\begin{itemize}
    \item $d_{in}=d_{out}=1$;
    \item $D^n=D^n(T)=\left\{ (x_i,T(x_i)) \right\}_{i=1}^n$ where $T(x)=2\tanh(x)-\tanh(4x)$ and it is called \textit{target function};
    \item the considered model is a 3-layer, 2-neuron MLP without bias terms, i.e.:
$$f(x;\boldsymbol{\theta})=v_1 \tanh(w_1 x)+v_2 \tanh(w_2 x).$$
\end{itemize}
In this same setting, the authors of \cite{Sato} showed that, when using online learning, for an optimal range of variance of the data sampling, the plateaus are enhanced due to \textit{noise-induced synchronization} \cite{noise_synch}. 

\subsection{Overfitting}
While the plateau phenomenon is an intrinsic structural problem for neural networks trained with gradient descent, overfitting is more closely related to learning excessive information from the training data.
In real-world applications, such unimportant information typically resides in observational noise, which was not modelled in the work of Fukumizu and Amari \cite{Amari}.
Thus, we would like to consider a dataset of the form $D=D^n_\tau(\omega;T)=\left\{(x_i(\omega),y_i(\omega))\right\}_{i=1}^n$, in which, given $\omega\in\Omega$ probability space:
\begin{eqnarray*}
    x_i&&\mathop{\sim}\limits^{\text{i.i.d.}}\rho \text{ (probability distribution on $\mathbb{R}^{d_{in}}$);}\\
    y_i(\omega) &&= \,\,T(x_i(\omega))+\xi_i(\omega)\,; \\
\xi_i&&\mathop{\sim}\limits^{\text{i.i.d.}}\mathcal{N}(0,\tau^2) \text{ (observational Gaussian noise).}
\end{eqnarray*}
When a target function $T$ is given, we can define the 
\textit{generalization error} as:
\begin{equation}
    \begin{aligned}
        R(\boldsymbol{\theta};T)=\|f(x;\boldsymbol{\theta})-T(x)\|^2_{L^2(\rho)}=\int_{\mathbb{R}^{d_{in}}}(f(x;\boldsymbol{\theta})-T(x))^2\rho(x)dx.
    \end{aligned}
\end{equation}
Note that:
\begin{equation}\label{eq:limits}
    \lim_{n\to\infty}\lim_{\tau\to0}L(\boldsymbol{\theta};D^n_\tau(\omega,T))=R(\boldsymbol{\theta};T)
\end{equation}
for any $\omega\in\Omega$, $\boldsymbol{\theta}\in\boldsymbol{\Theta}_m$. 
Introducing observational noise into the dataset leads to the problem of overfitting, which occurs when the model learns the noise in the data rather than the underlying structure of the target function and fails generalization. 
In other words, during the training process the generalization error may increase while the training error continues to decrease. 
In fact, even if our goal is to minimize $R$, we can only compute $L$, which depends on the dataset $D$ and therefore carries an intrinsic error that makes the learning trajectory deviate from the optimal one. 

\section{Generic Analysis}\label{sec:thms}
For the sake of simplicity, let us assume  $\sigma=\tanh$, $\tau>0$, 
and $T$ is an MLP of the form:
\[
T(x)=\sum_{i=1}^{m^*}v^*_i \tanh(w^*_i\cdot x+b_{1,i}^*)+b_2^*,
\]
where $w_l^*\neq w_j^*\neq 0$ for each $l,j=1,\dots,m^*$, $l\neq j$. To study vanishing gradients and overfitting rigorously, the first step is to define the objects of our analysis more formally. 
\begin{defin}
    The \textbf{optimal region} is the set of parameters that minimizes the generalization error:
    \begin{equation*}
        \mathcal{M}_m=\left\{\boldsymbol{\theta}\in\boldsymbol{\Theta}_m \mid R(\boldsymbol{\theta};T)=0\right\}.
    \end{equation*}
\end{defin}

\begin{defin}
        We call \textbf{overfitting region} the set of parameters that minimizes the training error:
    \begin{equation*}
        \mathcal{O}_m=\arg\min_{\boldsymbol{\theta}\in\boldsymbol{\Theta}_m} L(\boldsymbol{\theta};D)\,.
    \end{equation*}
\end{defin}

Note that $\mathcal{M}_m$ is fixed once the target function $T$ and the number of neurons $m$ are specified. 
On the other hand, $\mathcal{O}_m$ strongly depends on the dataset, and even when averaged over all possible finite datasets, it still varies with the variance $\tau$ of the observational noise. For $\tau=0$, it is clear that $\mathcal{M}_m\subseteq\mathcal{O}_m=\left\{\boldsymbol{\theta}\in\boldsymbol{\Theta}_m \mid L(\boldsymbol{\theta};D_0^n)=0\right\}$.
However, for any $\tau>0$, we have  $\mathcal{M}_m\cap\mathcal{O}_m=\emptyset$ almost surely, as shown in the following proposition.

\begin{prop}
    For every $m \geq m^*$, $\mathcal{M}_m$ does not contain any critical points of $L$ for almost every realization of the data noise vector $\xi=(\xi_1,\dots,\xi_n)$.\\
    Moreover, $L$ is constant over $\mathcal{M}_m$ and follows the distribution $\frac{\tau^2}{2n}\chi^2(n)$.
\end{prop}

\begin{proof}
    Let us rewrite the empirical risk as:
    \begin{equation}\label{eq:ER2}
        L(\boldsymbol{\theta};D) = \frac{1}{2n}\sum_{i=1}^n\left(f(x_i;\boldsymbol{\theta})-T(x_i)-\xi_i(\omega)\right)^2 .
    \end{equation}
    Taking the derivative of \eqref{eq:ER2} with respect to $\boldsymbol{\theta}$ yields:
    \begin{equation}\label{eq:derER}
    \begin{aligned}
        &\nabla_{\boldsymbol{\theta}} L(\boldsymbol{\theta};D) = \frac{1}{n}\sum_{i=1}^n\left(f(x_i;\boldsymbol{\theta})-T(x_i)-\xi_i(\omega)\right)\nabla_{\boldsymbol{\theta}} f(x_i;\boldsymbol{\theta}) .
    \end{aligned}
    \end{equation}
    If we take $\boldsymbol{\theta^*}\in\mathcal{M}$ and substitute in \eqref{eq:derER}, we get:
    \begin{equation}
    \begin{aligned}
        &\nabla_{\boldsymbol{\theta}} L(\boldsymbol{\theta};D) |_{\boldsymbol{\theta}=\boldsymbol{\theta^*}} = \frac{1}{n}\sum_{i=1}^n\left(T(x_i)-T(x_i)-\xi_i(\omega)\right)\nabla_{\boldsymbol{\theta}} f(x_i;\boldsymbol{\theta}) .
    \end{aligned}
    \end{equation}
    Since $\xi_i\sim \mathcal{N}(0,\tau^2)$, almost surely $\xi_i\neq0$ for every $i=1,\ldots,n$.
    Recalling that the derivative of the hyperbolic tangent is always positive, it holds that $\nabla_{\boldsymbol{\theta}} L(\boldsymbol{\theta};D)\neq0$ a.s. for all $\theta\in\mathcal{M}_m$, so the points in $\mathcal{M}_m$ cannot be critical points for $L$. 
    In addition, from \eqref{eq:ER2}, one can observe that $L(\boldsymbol{\theta};D_\tau^n(\cdot))\sim \frac{\tau^2}{2n}\chi^2(n)$ for every $\boldsymbol{\theta}\in\mathcal{M}_m$.
\end{proof}

An interpretation of the previous result is that, once observational noise is introduced, even if the data-generating mechanism is known, the resulting learning trajectory becomes unpredictable.
Numerical experiments suggest that the training process almost always (up to sets of measure zero) converges to a unique function. However, to the best of our knowledge, the literature provides no theoretical guarantee of convergence in this setting. With the intention of giving a formal statement supporting the experiments, we first want to prove that the gradient descent dynamics always converge in our setting.

\begin{prop}\label{prop:GD_convergence}
    In the Fukumizu-Amari setting, for any initial condition $\theta_0$ there exists a learning step $\eta$ such that the gradient descent algorithm either converges to a critical point of $L(\boldsymbol{\theta};D)$ or $\lim_{t\to\infty}\|\theta_t\|=+\infty$.
\end{prop}
\begin{proof}
    Referring to \cite{GD_convergence_3} (Theorem 3.2), we want to prove that $L(\theta)$ satisfies the hypothesis, i.e., it is analytical and exists $\kappa>0$ for which, for any choice of $\theta_0$ and for any $t>0$, $L$ satisfies the strong descent condition. In our case, it reads:
\begin{equation}\tag{SD}\label{eq:SD}
    L(\boldsymbol{\theta}(t))-L(\boldsymbol{\theta}(t+1))\geq \kappa \|\nabla_{\boldsymbol{\theta}}L(\boldsymbol{\theta}(t))\| \,\|\boldsymbol{\theta}(t+1)-\boldsymbol{\theta}(t)\|.
\end{equation}
The fact that $L(\boldsymbol{\theta};D)$ is analytical is clear, since it is a finite sum of analytical functions. 
To prove that \eqref{eq:SD} is true, let us first show that the following inequality holds: 
\begin{equation}\label{eq:SD_eq}
L(\boldsymbol{\theta}(t))-L(\boldsymbol{\theta}(t+1)) \geq \kappa\eta\|\nabla_{\boldsymbol{\theta}} L(\boldsymbol{\theta}(t))\|^2.
\end{equation}
By the properties of the gradient descent algorithm, since $L$ is smooth and $\nabla L(\boldsymbol{\theta};D)$ is Lipschitz , we can choose small enough $\eta$ such that at each step the error is lower than the previous iteration. This implies that the left-hand side is always non-negative and it's zero only if $\nabla_{\boldsymbol{\theta}} L(\boldsymbol{\theta}_{t})=0$.
Therefore, by choosing a small enough $\kappa$, the inequality holds for any $t>0$.
At this point, substituting $\|\boldsymbol{\theta}_{t+1}-\boldsymbol{\theta}_{t}\|=\eta\|\nabla_{\boldsymbol{\theta}} L(\boldsymbol{\theta}_{t})\|$ from \eqref{eq:GD} in \eqref{eq:SD_eq} will conclude the proof.

\end{proof}

\begin{cor}\label{cor:non-empty}
    Assume that $\lim_{t\to\infty}\|\boldsymbol{\theta}_t\|<+\infty$. Then, the overfitting region is not empty.
\end{cor}

\begin{proof}
    By contradiction, suppose that $L$ has no minimum points. This means that the dynamics always converges to saddle points, as the gradient descent algorithm cannot converge to maximum points. \\
    Define $\mu>0$ to be the infimum of $L$ over the saddle points. Let $\varepsilon>0$ and pick a saddle point $\boldsymbol{\alpha}\in\boldsymbol{\Theta}$ for which $L(\boldsymbol{\alpha};D)\leq\mu+\varepsilon$. 
    By definition of saddle point, there must be a direction $\boldsymbol{r}$ in which $L$ is decreasing. \\
    Let $\boldsymbol{\beta}\in\left\{\boldsymbol{\alpha}+a\boldsymbol{r} \mid\,a\in\mathbb{R}_+\right\}$ be the parameter that minimizes $L$. By taking $\varepsilon<\mu-L(\boldsymbol{\beta};D)$, we have that if  the gradient descent iteration is started from $\boldsymbol{\beta}$ we cannot converge to any saddle point. In fact, $L(\boldsymbol{\theta}(t))$ is non-increasing  in $t$ for small enough $\eta$ \cite{Bishop}.
    However, by Proposition \ref{prop:GD_convergence}, the algorithm has to converge to a critical point, but since $\mu$ was the infimum achieved by the training error on saddle points, this critical point must be a minimum.
\end{proof}

Note that if, instead of assuming $\lim_{t\to\infty}\|\boldsymbol{\theta}_t\|<+\infty$, one assumes that $\boldsymbol{\Theta}_m$ is compact, the result from Corollary \ref{cor:non-empty} is trivial: $L$ is a continuous function over a compact set and therefore there exists at least one minimum point.\\
The next theorem shows that, in a certain sense, the minimum point is unique. 

\begin{thm}\label{th:main}
    Assume $m(d_{in}+d_{out}+1)+d_{out} \le n$ and $\lim_{t\to\infty}\|\boldsymbol{\theta}_t\|<+\infty$.
    Then, for almost every realization of the noise vector $\xi$, there exists $r>0$ such that, with probability $1-\exp(-(\frac{r}{\tau}-\sqrt{n})^2/2)$, $\mathcal{O}_m$ consists of a single point up to the finite symmetry group generated by $(v_i,w_i) \mapsto (-v_i,-w_i)$, and neuron permutations.
\end{thm}

\begin{proof}
For the sake of simplicity, let us assume $d_{in}=d_{out}=1$ and $\boldsymbol{b}=\boldsymbol{b}^*=0$, so that dim$\boldsymbol{\Theta}_m=2m$. 
In the general case the same computations hold.\\
Define the parameter-output function $F:\boldsymbol{\Theta_m}\to\mathbb{R}^n$ as $F(\boldsymbol{\theta})=(f(x_1;\boldsymbol{\theta}),\dots,f(x_n;\boldsymbol{\theta}))$.
The Jacobian matrix of $F$ is:
\begin{equation*}
    J(\boldsymbol{\theta})= \begin{bmatrix}
        \tanh(w_1 x_1) & \cdots & v_m x_1 \operatorname{sech}^2(w_m x_1)\\
        \vdots & & \vdots \\
        \tanh(w_1 x_n) & \cdots & v_m x_n \operatorname{sech}^2(w_m x_n)
    \end{bmatrix}.
\end{equation*}
For distinct $\{x_j\}_{j=1}^n$ and generic $\boldsymbol{\theta}\in\boldsymbol{\Theta}_m$, the $2m$ columns are linearly independent 
when $n \ge 2m$. Hence, $J(\boldsymbol{\theta})$ has full rank $2m$ on a dense open set, 
so $F$ is an immersion. Thus $\mathcal{S} = F(\boldsymbol{\Theta}_m)$
is a $2m$-dimensional embedded real-analytic submanifold of $\mathbb{R}^n$. Thanks to \cite{leobacher2021uniqueness}, we know that
for any $C^2$ submanifold $M\subset\mathbb{R}^d$, there exists a maximal open set
$\mathscr{E}(M)$ such that every point $z\in \mathscr{E}(M)$ has a unique nearest point
$proj_M(z)\in M$, and the projection $proj_M:\mathscr{E}(M)\to M$ is $C^{1}$.\\
Namely, it can be proven that there exists a quantity $r>0$ called \emph{reach} such that each point with distance less than $r$ from $M$ has unique projection.
Since $F$ is a smooth embedding, the result holds in our case.

Let us call $y=(y_1,\dots,y_n)$. We can rewrite the risk as:
\[
    L(\boldsymbol{\theta}) = \frac{1}{2n}\|F(\boldsymbol{\theta}) - y\|^2 .
\]
By definition of orthogonal projection, the global minimum $\boldsymbol{\bar\theta}$ has to be the projection of the data point vector on $\mathcal{S}$, i.e., $\boldsymbol{\bar\theta} \in F^{-1}(proj_\mathcal{S}(y))$. In fact,
\[
\min_{\boldsymbol{\theta}\in\boldsymbol{\Theta}_m}L(\boldsymbol{\theta})=
\frac{1}{2n} \min_{\boldsymbol{\theta}\in\boldsymbol{\Theta}_m} \|F(\boldsymbol{\theta}) - y\|^2 = \frac{1}{2n} \min_{s\in\mathcal{S}} \|s - y\|^2,
\]
which is achieved by the projection of $y$ on $\mathcal{S}$ (we have proven in Corollary \eqref{cor:non-empty} that $L$ has minimum over $\boldsymbol{\Theta}_m$).\\
Note that, in this view, we can rewrite $y=F(\boldsymbol{\theta}^*)+\xi$ where.
Therefore, we can say that if $\|\xi\|\leq r$, then $y\in\mathscr{E}(\mathcal{S})$ and it is well defined $proj_\mathcal{S}(y)=s^*$.\\
Since $\xi\sim\mathcal{N}(0,\tau^2Id(n))$, we know from the Gaussian tails bound that $$\mathbb{P}[\|\xi\| < r]\geq 1-\exp(-(\frac{r}{\tau}-\sqrt{n})^2/2),$$ which goes to $1$ for $\tau\to0$ or $n\to\infty$. 
To conclude, we only need to show that $F^{-1}(s^*)$ is unique up to symmetries.
To this aim, recall that $F$ fixes the value of $f(\cdot;\boldsymbol{\theta})$ in $n\geq2m$ points, therefore by the interpolation theory we can say that $f(\cdot;\boldsymbol{\theta})$ is unique.
Using the results from \cite{Chen1993}, the fact that $f(\cdot;\boldsymbol{\theta})$ is unique directly implies that the parameter $\boldsymbol{\theta}$ is also unique up to neuron symmetries and transformations of the form $(v_i,w_i) \mapsto (-v_i,-w_i)$.
\end{proof}

In our vision, it is beyond the scope of this article to discuss precise probability estimates for the previous theorem. However, general bounds for the reach can be found in \cite{Aamari2019}. One interpretation of this result is that, under certain conditions (that can be achieved, for example, by choosing small data noise variance or a large sample set), up to a measure zero set (to be specified below) every initial condition will converge to an overfitting point, which is unique if we consider the space of functions. 
To better clarify what this negligible set is and to give more insights about the plateau dynamics that appear during the training, let us recall an important statement from \cite{optimal_structure}:

\begin{thm}[Simsek et al. (2021)]
    For an irreducible critical point (for $L$) $\boldsymbol{\theta}_r^\ast\in\boldsymbol{\Theta}_r$, the set
    \[
    \bar{\Theta}_{r \to m}(\boldsymbol{\theta}_r^\ast)
    \]
    is a union of
    $G(r,m)$
    distinct, non-intersecting affine subspaces of dimension $m-r$ and
    all points in $\bar{\Theta}_{r \to m}(\boldsymbol{\theta}_r^\ast)$ are critical points of $L$. 
    Moreover, for $C^2$ functions $h$ and $\sigma$, if $\boldsymbol{\theta}_r^\ast$ is a strict saddle, then all points in
    $\bar{\Theta}_{r \to m}(\boldsymbol{\theta}_r^\ast)$ are also strict saddles.
\end{thm}
In the above citation, a parameter $\alpha\in\boldsymbol{\Theta}_m$ is said to be \emph{irreducible} if there does not exist any $q<m$, $\beta\in\boldsymbol{\Theta}_m$ such that $f(\cdot;\alpha)=f(\cdot;\beta)$. 
In addition, given $\boldsymbol{\bar\theta}\in\mathcal{O}_r$ we define the \emph{symmetry-induced
critical points} of $\boldsymbol{\bar\theta}$ as:
\begin{equation*}
    \bar\Theta_{r\to m}(\boldsymbol{\bar\theta})=\left\{\boldsymbol{\theta}\in\boldsymbol{\Theta}_{m} \mid f(x;\boldsymbol{\theta})=f(x;\boldsymbol{\bar\theta})\wedge\boldsymbol{\theta}_i\neq0\,\forall i\right\} .
\end{equation*}

The measure zero set for which the parameters do not converge to the global minimum is the set of parameters that contain $0<\mu<m$ \emph{completely synchronized} neurons, i.e., there exist $\mu$ pairs of $i\neq j$, $i,j\in[m]$ such that $\boldsymbol{\theta}_i=\boldsymbol{\theta}_j$. 
In fact, for such neurons the gradient iteration is the same and so they have no chance of separating as time evolves, causing the whole dynamic to live in the subspace $\boldsymbol{\Theta}_{m-\mu}$ (and therefore to converge to $\mathcal{O}_{m-\mu}$). 
If the neurons are not completely synchronized, such subspace is unreachable in finite time \cite{zhang2026saddletosaddle}.
For a more in-depth analysis of the dynamical properties of the embedded critical points $\bar\Theta_{r\to m}(\mathcal{O}_{m-\mu})$, one can refer to \cite{fukumizu_semi-flat_2019}. With the theoretical ground provided by the previous results, we expect that for small data noise variance the learning dynamics, after visiting some saddle points including the plateau regions  and the optimal region $\mathcal{M}_m$, will eventually converge to a point in the overfitting region  $\mathcal{O}_m$, which is a set of pointwise attractors. 
A schematic view of this saddle–saddle–attractor scenario, representing the trajectory of the parameters, is depicted in Fig.~\ref{fig:schematic}. 
\begin{figure}[hbt]
\centering
\includegraphics[width=0.7\linewidth]{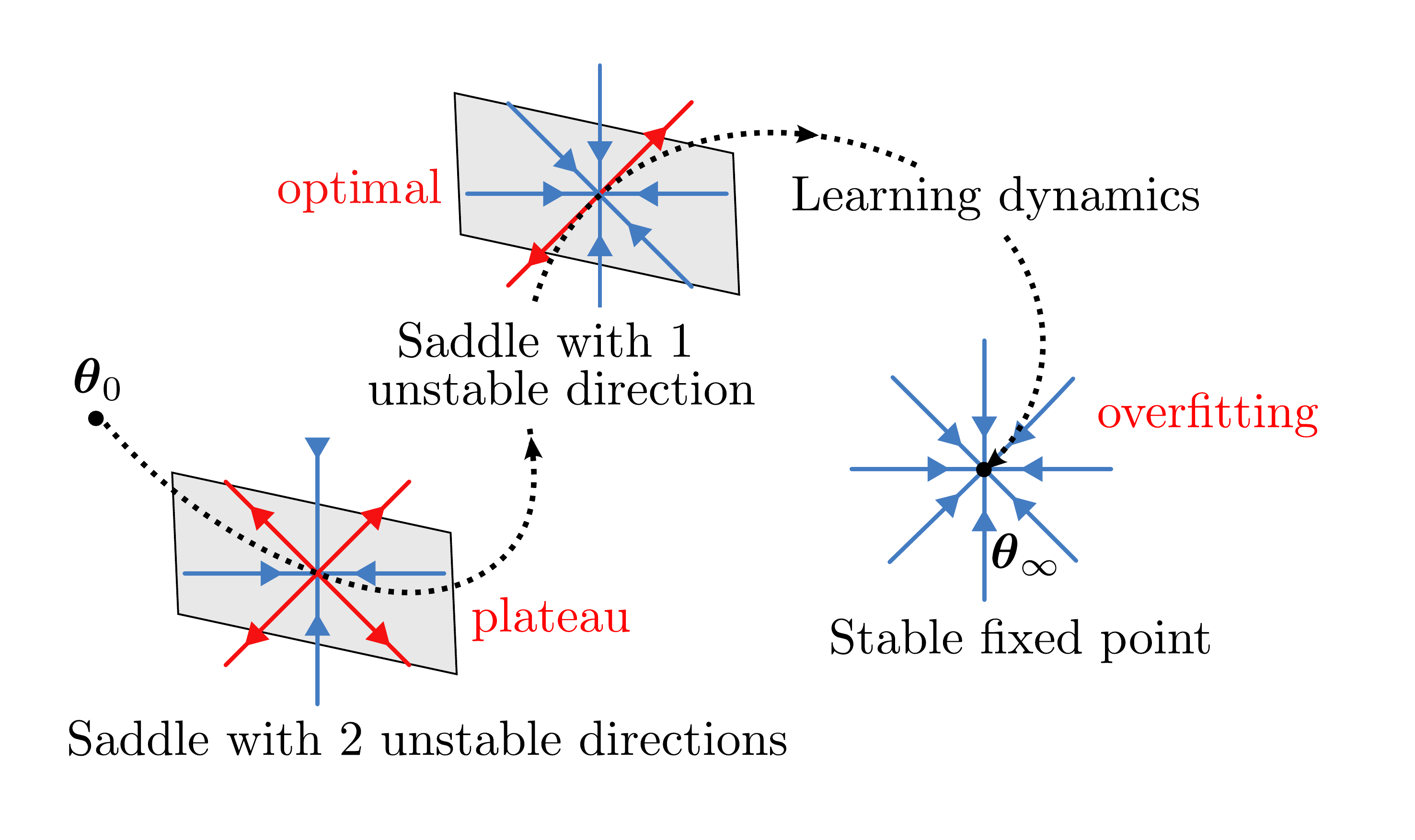}
\caption{A schematic representation of the saddle-saddle-attractor scenario in MLP gradient descent learning. Empirically, the number of positive eigenvalues is smaller near the optimal region than in the plateau region. The overfitting is a  stable attractor.}
\label{fig:schematic}
\end{figure}

\section{Minimal model and numerical experiments}\label{sec:min_model}

To confirm the saddle–saddle–attractor scenario illustrated in Fig.~\ref{fig:schematic} through numerical experiments, we propose a minimal MLP model that captures both vanishing gradient and overfitting phenomena. 
Our idea is to use a modified Fukumizu-Amari model:
\begin{itemize}    
    \item MLP with 1-input, 1-output, 3-layer, 2-neuron without bias;
$$f(x;\boldsymbol{\theta})=v_1 \tanh(w_1 x)+v_2 \tanh(w_2 x),$$
\item Fixed observational noise; $D=D_\tau^n(\omega)$ 
\item Target function;  $T(x)=2\tanh(x)$;
\end{itemize}
By means of this model, we will be able to represent the significant dynamics of the learning to give a clear view of the saddle-saddle-attractor scenario. From a dynamical system's perspective, we are focusing on the following 4-dimensional quenched random  map:
\begin{equation*}
\begin{aligned}
w_1(t+1)&=w_1(t)- \frac{\eta v_1(t)}{n} \sum_{i=1}^n \frac{x_ih(x_i,y_i;\boldsymbol{\theta})}{\cosh^2(w_1(t)x_i)},\\
w_2(t+1)&=w_2(t)-\frac{\eta v_2(t)}{n} \sum_{i=1}^n \frac{x_ih(x_i,y_i;\boldsymbol{\theta})}{\cosh^2(w_2(t)x_i)}, \\
v_1(t+1)&=v_1(t) - \frac{\eta}{n} \sum_{i=1}^n \tanh(w_1(t)x_i)h(x_i,y_i;\boldsymbol{\theta}), \\
v_2(t+1)&=v_2(t) - \frac{\eta}{n} \sum_{i=1}^n \tanh(w_2(t)x_i)h(x_i,y_i;\boldsymbol{\theta}) .
\end{aligned}
\end{equation*}

\noindent
where $h(x,y;\boldsymbol{\theta})=\|f(x;\boldsymbol{\theta})-y\|_2$ and $D=\left\{(x_i,y_i)\right\}_{i=1}^n$.

Note that in this simple setting we know that the singular region producing plateaus is: $\bar\Theta_{0\to 2}(0)\cup\bar\Theta_{1\to 2}(\mathcal{O}_1)$ and the optimal region is the embedding of $T$ in $\boldsymbol{\Theta
}_2$.

\subsection{Numerical experiments}
We present numerical evidence supporting the dynamical scenario of Fig.~\ref{fig:schematic} within the minimal model. We generate $n=100$ samples $x_i\sim\mathcal{N}(0,1)$ and construct two datasets, $D_0^{100}$ and $D_{0.2}^{100}$. Using the four-dimensional map introduced above, we perform gradient descent up to $t=2\times10^6$. For these parameters, the assumptions of Theorem \ref{th:main} are empirically satisfied, and all trajectories converge to the same solution. Figure \ref{fig:experiments} shows the learning curves (training and generalization errors in log–log scale) and parameter trajectories in $\boldsymbol{\Theta}_2$. The dynamics follow the schematic scenario: trajectories first approach a singular region, producing a plateau, then move toward near-optimal states with slow evolution, and finally escape to an overfitting solution. A key difference between $D_0^{100}$ and $D_{0.2}^{100}$ is that, in the noisy case, the training error remains flat near the optimum, whereas in the noiseless case it decreases again. Consistent with \cite{spectral_bias}, the early-stage dynamics are nearly identical in both cases, indicating that noise is learned only at later stages and ultimately drives overfitting. We also consider an extended setting with two neurons ($\boldsymbol{\theta}\in\boldsymbol{\Theta}_4$), where the same qualitative scenario persists. Spectral analysis suggests that the near-optimal singular region has fewer unstable directions than other degenerate regions. In particular, along the plateau the Hessian exhibits two positive eigenvalues, whereas near the optimal region it typically has only one, indicating reduced instability.


\begin{figure*}[t]
    \centering

    \begin{subfigure}{\textwidth}
        \centering
        \begin{subfigure}{0.48\textwidth}
            \centering
            \includegraphics[width=\linewidth]{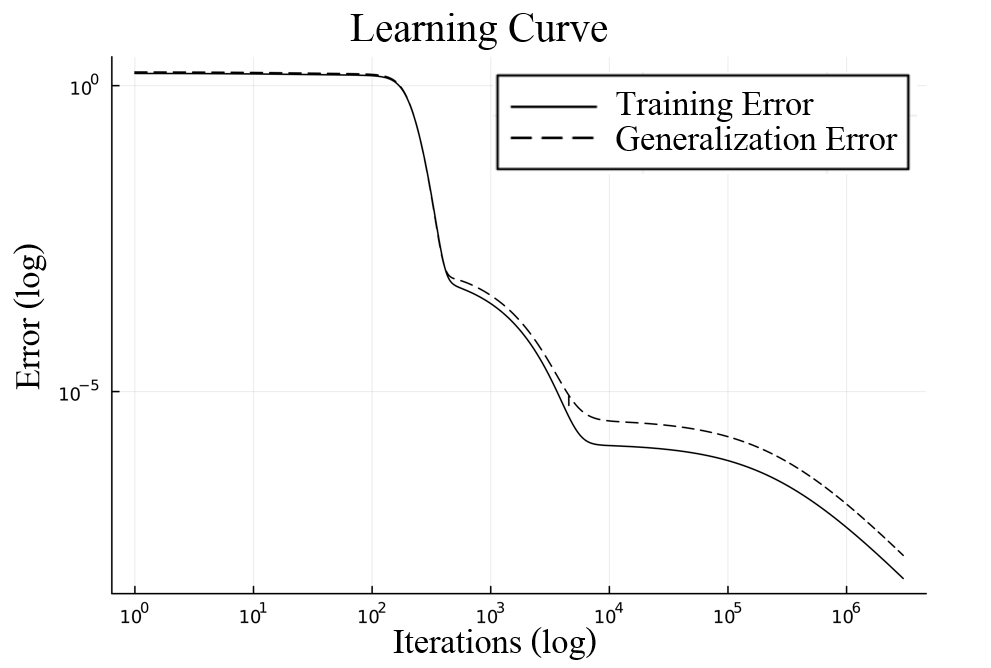}
        \end{subfigure}
        \hfill
        \begin{subfigure}{0.5\textwidth}
            \centering
            \includegraphics[width=\linewidth]{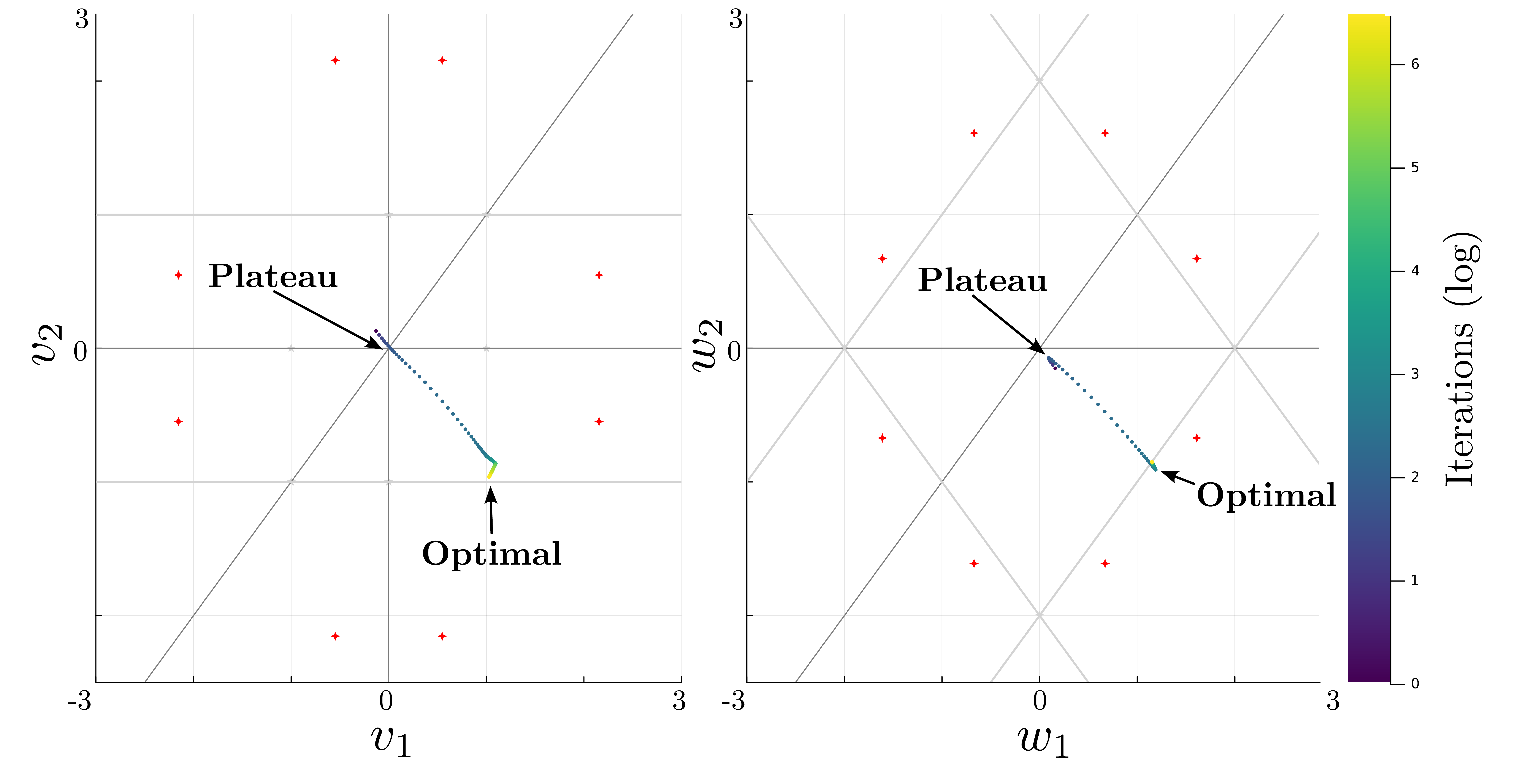}
        \end{subfigure}
        \caption{ }
    \end{subfigure}

    \vspace{0.1cm}

    \begin{subfigure}{\textwidth}
        \centering
        \begin{subfigure}{0.48\textwidth}
            \centering
            \includegraphics[width=\linewidth]{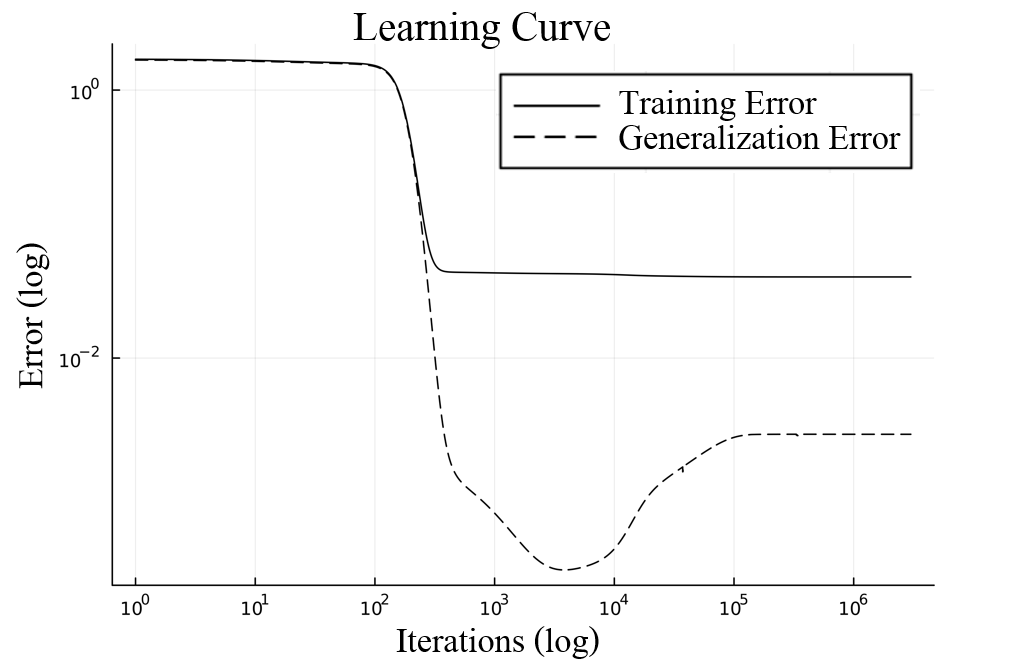}
        \end{subfigure}
        \hfill
        \begin{subfigure}{0.5\textwidth}
            \centering
            \includegraphics[width=\linewidth]{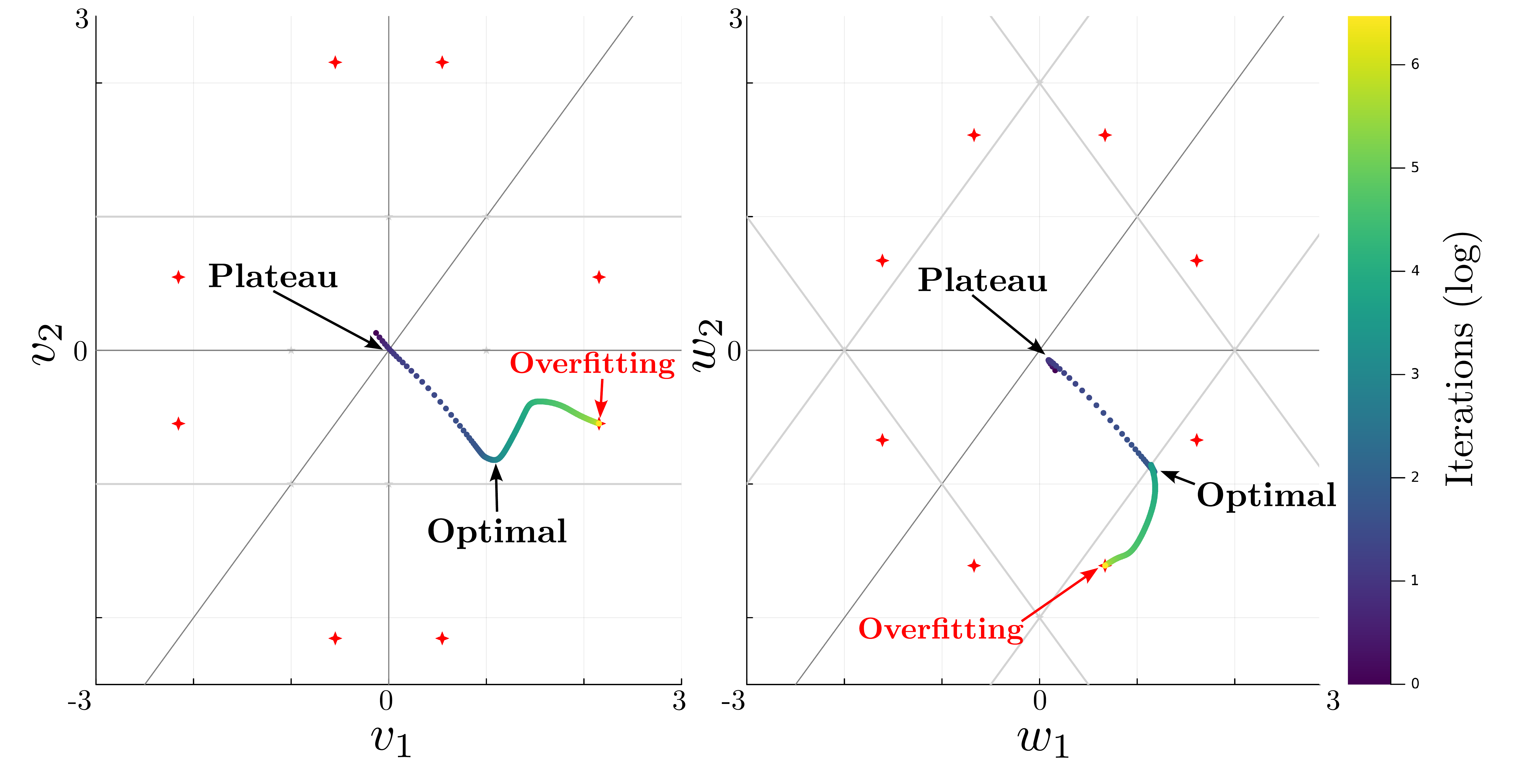}
        \end{subfigure}
        \caption{ }
    \end{subfigure}

    \caption{Graphs obtained after training the minimal model for 2 million iterations. In the first column we represent the learning curve; on the right there is a representation of the parameters' orbits during the training, where red stars are the overfitting points and grey lines model the singular region (dark grey) and optimal region (light gray). Plots in (a) represent the training without observational noise (i.e., $\tau=0$), while in (b) we added some small noise with $\tau=0.2$.}
    \label{fig:experiments}
\end{figure*}

\section{Conclusion}

Vanishing gradients and overfitting in neural networks are often studied in highly complex settings, where the underlying mechanisms remain obscured. Here we instead adopt a minimal framework that isolates the essential dynamics. In Section \ref{sec:min_model}, we introduce a two-neuron, bias-free multi-layer perceptron that approximates a one-neuron MLP from noisy observations and exhibits both plateaus and overfitting. Building on this setting, we prove in Section \ref{sec:thms} that, under mild assumptions, the nonlinear learning dynamics converge almost surely to a unique point modulo the symmetries of tanh networks. We further show (Fig.~\ref{fig:experiments}) that the stability of the optimal region is altered by observational noise: while the target parameters are attractors in the noiseless case, they become saddles for any $\tau>0$. Empirically, these saddles appear more attracting than those associated with other singular regions, leading to a saddle-saddle-attractor scenario. Despite its simplicity, this framework raises several open questions. In particular, it remains to characterize sharper conditions for uniqueness in Theorem \ref{th:main}, and to quantify the distance $\delta$ between optimal and singular regions as a function of $\tau$, which may guide systematic refinement of early-stopped solutions.

\section*{Acknowledgements}
Authors thank Professor Andrea Agazzi and Mr.~Taichi Yano.
The research leading to these results has been partially
funded by the Grant in-Aid
for Scientific Research (B) No. 21H01002, JSPS,
Japan. 

\section*{Declaration of competing interest}
The authors declare that they have no known competing financial interests or personal relationships that could have appeared to influence the work reported in this paper.

\section*{CRediT authorship contribution statement}
\textbf{Alex Alì Maleknia}: Writing – original draft, Software, Investigation, Formal analysis, Data curation. 
\textbf{Yuzuru Sato}: Writing – review \& editing, Validation, Investigation, Funding acquisition, Conceptualization.

\section*{Research data}
All the necessary tools and information to reproduce the experiments presented in this article can be found at:
\hyperlink{link}{https://github.com/alexmare2/Dynamical-structure-of-MLP-with-overfitting}.




\bibliographystyle{elsarticle-num}
\bibliography{refs}







\end{document}